\documentclass[a4paper]{article}
\usepackage{graphicx}
\usepackage{subfigure}
\usepackage{amsmath,amsthm}
\usepackage{amssymb,amsfonts}
\usepackage{authblk}
\usepackage{url}

\newtheorem{definition}{\textbf{Definition}}

\newtheorem{theorem}{\textbf{Theorem}}
\newtheorem{proposition}{\textbf{Proposition}}
\newtheorem{corollary}{\textbf{Corollary}}

\title{Estimating Transfer Entropy via Copula Entropy}
\author{Jian MA\thanks{Email: majian@hitachi.cn}}
\affil{Hitachi (China) Research \& Development Corporation}

\begin{document}

\maketitle

\begin{abstract}
	\noindent
	Causal discovery is a fundamental problem in statistics and has wide applications in different fields. Transfer Entropy (TE) is a important notion defined for measuring causality, which is essentially conditional Mutual Information (MI). Copula Entropy (CE) is a theory on measurement of statistical independence and is equivalent to MI. In this paper, we prove that TE can be represented with only CE and then propose a non-parametric method for estimating TE via CE. The proposed method was applied to analyze the Beijing PM2.5 data in the experiments. Experimental results show that the proposed method can infer causality relationships from data effectively and hence help to understand the data better. 
\end{abstract}

\section{Introduction}
Causality is about the relationship between cause and effect and is ubiquitous in natural and social worlds. Questing such relationship is the central topic of different sciences. Causal discovery \cite{holland1986statistics} is the statistical problem to identify such causal relations from observational or experimental data collected from the underlying systems. In statistics, association is closely related to causality. However, the former does not imply the latter, as is well known that correlation does not mean causation. To discover causality relationships, conceptual tools beyond association are needed.

Granger \cite{granger1969investigating,granger1980testing} developed the notion of causality between two stationary time series based on the philosophy that cause should improve the prediction of effects. Given two random variable $X_t, Y_t, t = 1,\ldots,T$, Granger Causality (GC) is defined as follows:
\begin{definition}[Granger Causality]
	\begin{equation}
	\label{eq:gc}
		\delta^2(Y_{t+1}|Y_t,X_t) < \delta^2(Y_{t+1}|Y_t),
	\end{equation}
	if $X$ causes $Y$, where $\delta^2$ denotes variance.
\end{definition}

Transfer Entropy (TE) is another information theoretic measure of causality defined for stationary time series by Schreiber \cite{schreiber2000measuring}. Inspired by the notion of Mutual Information (MI), TE is defined in the form of conditional MI, as follows:
\begin{definition}[Transfer Entropy]
	TE between time series $X_i,Y_i, i=1,\ldots,T$ is 
	\begin{equation}
	\label{eq:te}
		TE = \sum{p(Y_{i+1},Y^i,X_i)\log{\frac{p(Y_{i+1}|Y^i,X_i)}{p(Y_{i+1}|Y^i)}}},
	\end{equation}
	where $Y^i=(Y_1,\ldots,Y_i)$.
\end{definition}
\noindent
The above definition can be written as condition MI, as follows:
\begin{equation}
	TE = I(Y_{i+1};X_i|Y^i).
\end{equation}
\noindent
By such definition, TE is a model-free measure and can be interpreted as the information cause provides to effects. TE assumes time series to be stationary. Additionally, since conditional probability are used in the definition, it means TE also assumes Markovianity. Though having great potential applications to many fields, TE has been considered notoriously difficult to estimate \cite{zhang2011uai}. 

Both GC and TE are defined based on the same philosophy of causality. Barnett et al. have shown that GC and TE are equivalent under the Gaussianity assumption \cite{barnett2009granger}. Apparently, TE can be applied to more broad cases than GC. Directed Information is a notion closely related to TE, which is defined as the sum of a group of TE and one MI \cite{massey1990causality}. 

Copula theory is about the representation of statistical dependence \cite{nelsen2007introduction,joe2014dependence}. Copula Entropy (CE) is a recently introduced theory on measurement of statistical independence \cite{ma2011mutual}. CE has been proved to be equivalent to MI \cite{ma2011mutual} and therefore connects copula theory with information theory. It is a ideal measure of statistical independence since it has several good properties, such as multivariate, symmetric, non-negative (0 iff independent), invariant to monotonic transformation, equivalent to correlation coefficient in Gaussian cases. In \cite{ma2011mutual}, Ma and Sun also proposed a non-parametric method for estimating CE (See Section \ref{s:est}). CE has been applied to discover associations in data \cite{ma2019discovering}. 

In this paper, we propose a non-parametric method for estimating TE with CE. Specifically, we first prove that TE can also be represented with only CE and hence propose a method for estimating TE via CE based on this representation. We test the proposed method on the Beijing PM2.5 data to identify the causality in environmental and meteorological systems and also compare it with the related methods.

\section{Related Research}
Estimating TE is a fundamental problem for its applications. One of the basic estimation methods is based on the definition of TE. Faes et al. \cite{faes2013compensated} used the original definition of TE to estimate it as the difference between two conditional entropies. Vicente et al. \cite{vicente2011transfer} proposed to estimate TE by expanding the definition of TE into the sum of four entropies. Kontoyiannis and Skoularidou \cite{kontoyiannis2016estimating} proposed to estimate directed information via likelihood based MI estimator and showed the asymptotic convergence of such estimator theoretically.

Copula has been applied to estimate measures of causality. Taoufik et al. represented the hypothesis of Conditional Independence (CI) in terms of copula densities and then defined the test statistic based on the Hellinger distance between two terms of copula densities \cite{bouezmarni2012nonparametric}. Hu and Liang proposed to compute GC with copula, which write the definition of GC into a conditional copula density formation and then estimate conditional copula density empirically \cite{hu2014copula}. This work proposed to estimate GC with non-parametric PDF approximation, which is lack of convergence guarantee. Song \cite{song2009testing} proposed to test the hypothesis of CI after deriving conditional copula with Rosenblatt transforms. No\"el et al. \cite{veraverbeke2011estimation} proposed the estimators of conditional association measures with the estimated conditional copula. The above three research are all testing CI based on conditional copula. However, estimating conditional copula from data is usually biased and unstable. Since TE is essentially a conditional MI, it is natural to link it to CE. Wieczorek and Roth \cite{wieczorek2016causal} proposed the notion of causal compression with CE, which proved that Directed Information can be represented with only CE. Kim et al. proposed a Copula Nonlinear GC with a VAR system by using a copula version of beta regression model \cite{kim2020copula}. Cui et al. considered learning causal structure from data with missing values with Gaussian copula assumption \cite{cui2019learning}. Reddi and P\'oczos \cite{reddi2013scale} proposed a scale invariant conditional dependence measure by transforming Hilbert-Schmidt dependence measures with copula functions. Testing CI through partial copula was a popular topic recently. Bergsma \cite{bergsma2004testing,bergsma2010nonparametric} first proposed to test CI by testing independence between the variables derived from original variable by partial copula transformation. Bianchi, et al. \cite{bianchi2020conditional} introduced a test statistic for CI called weighted partial copula which is defined as the weighted distance between the estimated conditional (or partial) copula and independence copula. Also with partial copula, Petersen and Hansen \cite{petersen2020testing} proposed a test statistics based on generalized correlation between the residuals estimated with quantile regression. Frattarolo and Guegan \cite{frattarolo2013empirical} proposed to test CI with empirical conditional copula projections.

Another line of the  related research is on estimating Conditional MI (CMI) based on the kNN method for estimating MI \cite{frenzel2007partial,vejmelka2008inferring,poczos2012nonparametric,runge2018conditional}. TE is essentially CMI. These related works on estimating CMI are all based on a same idea that expands CMI into four terms of entropies by definition and then estimates each term with kNN entropy estimator. In this paper, we will proposed a different kNN-based method for estimating TE (or CMI). 

In this paper, we will proposed a method for estimating TE via CE based on the representation TE with only CE. Together with the previous work \cite{ma2011mutual}, we will develop a theoretical framework on testing both independence and CI based on CE. Previously, two similar frameworks for (conditional) independence testing were also proposed based on kernel tricks in machine learning \cite{gretton2007nips,zhang2011uai} and distance covariance/correlation \cite{szekely2007measuring,szekely2009brownian,wang2015conditional}. Both frameworks can be considered as nonlinear generalization of traditional (partial) correlation, and have non-parametric estimation methods. The kernel-base framework is based on the idea, called kernel mean embedding, that test correlation \cite{gretton2007nips} or partial correlation \cite{zhang2011uai} by transforming distributions into RKHS with kernel functions. Another framework defines a concept called distance correlation with characteristic function \cite{szekely2007measuring,szekely2009brownian}. With this concept, Wang et al. \cite{wang2015conditional} proposed a new concept for testing CI, called Conditional Distance Correlation (CDC), defined with characteristic function for conditional functions.

\section{Copula Entropy}
\label{s:CopEnt}
\subsection{Theory}
Copula theory is about the representation of multivariate dependence with copula function \cite{nelsen2007introduction,joe2014dependence}. At the core of copula theory is Sklar theorem \cite{sklar1959fonctions} which states that multivariate probability density function can be represented as a product of its marginals and copula density function which represents dependence structure among random variables. Such representation separates dependence structure, i.e., copula function, with the properties of individual variables -- marginals, which make it possible to deal with dependence structure only regardless of joint distribution and marginal distributions. This section is to define an statistical independence measure with copula. For clarity, please refer to \cite{ma2011mutual} for notations.

With copula density, Copula Entropy is define as follows \cite{ma2011mutual}:
\begin{definition}[Copula Entropy]
\label{d:ce}
	Let $\mathbf{X}$ be random variables with marginal distributions $\mathbf{u}$ and copula density $c(\mathbf{u})$. CE of $\mathbf{X}$ is defined as
	\begin{equation}
	H_c(\mathbf{X})=-\int_{\mathbf{u}}{c(\mathbf{u})\log{c(\mathbf{u})}}d\mathbf{u}.
	\end{equation}
\end{definition}

In information theory, MI and entropy are two different concepts \cite{cover1999elements}. In \cite{ma2011mutual}, Ma and Sun proved that they are essentially same -- MI is also a kind of entropy, negative CE, which is stated as follows: 
\begin{theorem}
\label{thm1}
	MI of random variables is equivalent to negative CE:
	\begin{equation}
	I(\mathbf{X})=-H_c(\mathbf{X}).
	\end{equation}
\end{theorem}
\noindent
The proof of Theorem \ref{thm1} is simple \cite{ma2011mutual}. There is also an instant corollary (Corollary \ref{c:ce}) on the relationship between information of joint probability density function, marginal density function and copula density function, which is stated as follows:
\begin{corollary}
\label{c:ce}
	\begin{equation}
		H(\mathbf{X})=\sum_{i}{H(X_i)}+H_c(\mathbf{X}).
	\end{equation}
\end{corollary}
The above results cast insight into the relationship between entropy, MI, and copula through CE, and therefore build a bridge between information theory and copula theory. CE itself provides a mathematical theory of statistical independence measure.

\subsection{Estimation}
\label{s:est}
It has been widely considered that estimating MI is notoriously difficult. Under the blessing of Theorem \ref{thm1}, Ma and Sun \cite{ma2011mutual} proposed a simple and elegant non-parametric method for estimating CE (MI) from data which composes of only two steps:\footnote{The \textbf{copent} packages in \texttt{R} and \texttt{Python} are available on CRAN and PyPI, and also on GitHub at \url{http://github.com/majianthu/copent}.}
\begin{enumerate}
	\item Estimating Empirical Copula Density (ECD);
	\item Estimating CE.
\end{enumerate}

For Step 1, if given data samples $\{\mathbf{x}_1,\ldots,\mathbf{x}_T\}$ i.i.d. generated from random variables $\mathbf{X}=\{x_1,\ldots,x_N\}^T$, one can easily estimate ECD as follows:
\begin{equation}
F_i(x_i)=\frac{1}{T}\sum_{t=1}^{T}{\chi(\mathbf{x}_{t}^{i}\leq x_i)},
\end{equation}
where $i=1,\ldots,N$ and $\chi$ represents for indicator function. Let $\mathbf{u}=[F_1,\ldots,F_N]$, and then one can derive a new samples set $\{\mathbf{u}_1,\ldots,\mathbf{u}_T\}$ as data from ECD $c(\mathbf{u})$. In practice, Step 1 can be easily implemented non-parametrically with rank statistics.

Once ECD is estimated, Step 2 is essentially a problem of entropy estimation which has been contributed with many existing methods. Among them, the kNN method \cite{kraskov2004estimating} was suggested in \cite{ma2011mutual}. With rank statistics and kNN methods, one can derive a non-parametric method of estimating CE, which can be applied to any situation without assumption on the underlying system.

\section{TE via CE}
In this section, we propose a method for estimating TE via CE. Before that, we derive a representation of TE with CE.
\begin{proposition}
	TE can be represented with only CE.
\end{proposition}
\begin{proof}
The proof starts from the definition of TE \eqref{eq:te}. After expanding the definition equation, the definition can be easily transformed into the equation composed of four different CE terms.
\begin{align}
	TE &= \sum{p(Y_{i+1},Y^i,X_i)\log{\frac{p(Y_{i+1}|Y^i,X_i)}{p(Y_{i+1}|Y^i)}}}\\
	&=\sum{p(Y_{i+1},Y^i,X_i)\log{\frac{p(Y_{i+1},Y^i,X_i)p(Y^i)}{p(Y_{i+1},Y^i)p(Y^i,X_i)}}}\\
	&=I(Y_{i+1};Y^i;X_i) - I(Y_{i+1};Y^i) - I(Y^i;X_i)\\
	&=-H_c(Y_{i+1};Y^i;X_i) + H_c(Y_{i+1};Y^i) + H_c(Y^i;X_i) \label{eq:three1}\\
	&=-H_c(Y_{i+1},Y^i,X_i) + H_c(Y_{i+1},Y^i) + H_c(Y^i,X_i) - H_c(Y^i). \label{eq:three}
\end{align}
The last term $H_c(Y^i)$ in \eqref{eq:three} equals to 0 if $Y^i=Y_i$.
\end{proof}
The above proposition shows that TE is the sum of four terms: joint CE of cause $X$ and the past and future of effect $Y$, self-joint CE of $Y$, association between cause $X$ and effects $Y$, and joint CE of the past of $Y$ (exists only if it is multivariate). The second term means excluding the information of self dynamics of effects from joint CE and the third term means excluding the information from cause to the past of effects.

With this representation, we propose a method for estimating TE via CE by two simple steps: 
\begin{enumerate}
	\item estimating the three or four CE terms in \eqref{eq:three}; 
	\item calculating TE from these estimated CEs. 
\end{enumerate}
CE can be estimated with the method in Section \ref{s:est}, and hence we derive a non-parametric method for estimating TE. The proposed method inherits the merits of the method for estimating CE, including model-free, tuning-free, good convergence performance and low computation burden.

\section{Experiments and Results}
\subsection{Data}
The Data used in our experiment is the Beijing PM2.5 dataset in the UCI machine learning repository \cite{asuncion2007uci}, which is about air pollution at Beijing. This hourly data set contains the PM2.5 data of US Embassy in Beijing. Meanwhile, meteorological data from Beijing Capital International Airport are also included. The data has been analyzed at month scale \cite{liang2015assessing,kreuzer2019bayesian}.

Meteorological factors in data include dew point, temperature, pressure, cumulated wind speed, combined wind direction, cumulated hours of snow, cumulated hours of rain. The first four factors are analyzed in our experiments. The data was collected hourly from Jan. 1st, 2010 to Dec. 31st, 2014, which results in 43824 samples with missing values. To avoid tackling missing values, only the data from April 2nd, 2010 to May 14th, 2010 were used in our experiments, which contains 1000 samples without missing values. 

In our experiments, we analyze the causal relationship between meteorological factors and PM2.5 at hour scale. Studying such relationship can help to understand the underlying mechanism of pollution generation and to build the forecasting model of PM2.5.

\subsection{Experiments}
\footnote{The code in \texttt{R} and \texttt{Python} for the experiments is available on GitHub at \url{https://github.com/majianthu/transferentropy}.}
In the experiments, we estimated TE between meteorological factors and PM2.5 to measure how the former affect the latter after several hours lag. In the experiments, for \eqref{eq:three}, two cases are considered: $Y^i=Y_i$ and $Y^i=(Y_i,Y_{i-1},Y_{i-2},Y_{i-3})$, which means the past one and four states are conditioned. The latter case is for testing the Markovianity of time series data. The time lags in the experiments is from 1 to 24 hours. To investigate the relationship between three CE terms in TE, we also estimated them from data respectively. To investigate the dynamic relationship between meteorological factors, we estimated TE between four pairs of factors: temperature to pressure and dew point, and cumulated wind speed to temperature and pressure. TE is estimated with the proposed non-parametric method. The non-parametric method for estimating CE is implemented in the \texttt{R} package '\textsf{copent}' on CRAN \cite{ma2020copent}. 

We also conducted an experiment to compare our method with two other methods on testing CI: kernel-based CI (KCI) \cite{zhang2011uai} and CDC \cite{wang2015conditional}. The three methods were compared on inferring causal relationships between meteorological factors and PM2.5 from data. The \texttt{R} packages '\textsf{CondIndTests}' and '\textsf{cdcsis}' are used in the experiments as the implementations of KCI and CDC respectively.

\subsection{Results}
The estimated TE of four factors are shown in Figure \ref{f:te}. It can be learned from the Figure that TE of the four meteorological factor increase sharply in the first 9 hours time lags and that TE of temperature and cumulated wind speed reach to their peak value at 9 hours time lags while TE of the other two factors still increase but with relatively slow rate. Generally speaking, the trends of TE of dew point and pressure are similar and that of the other two factors are similar.

\begin{figure}
	\centering
	\subfigure[Dew Point]{\includegraphics[width=0.45\textwidth]{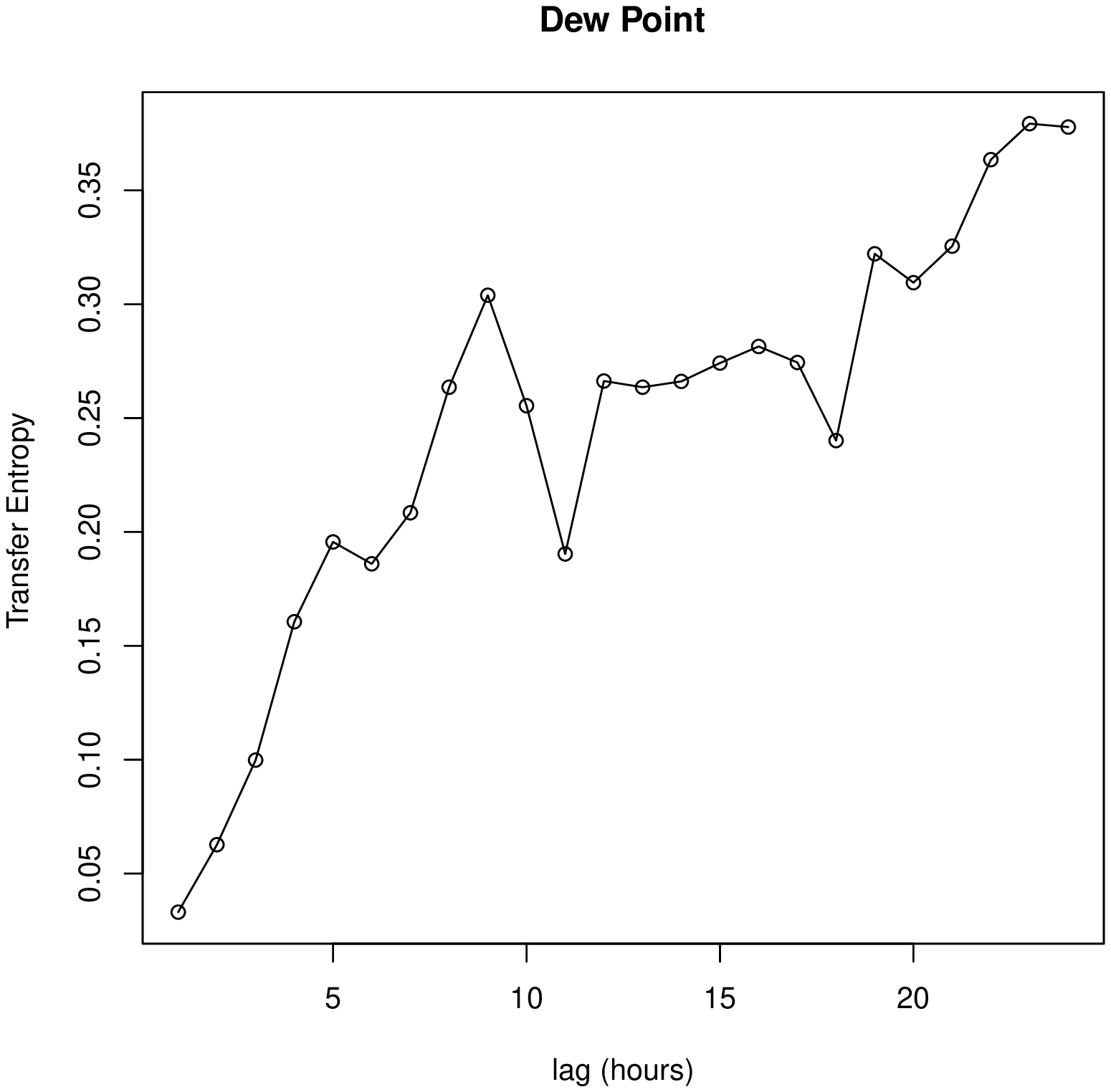}}	
	\subfigure[Temperature]{\includegraphics[width=0.45\textwidth]{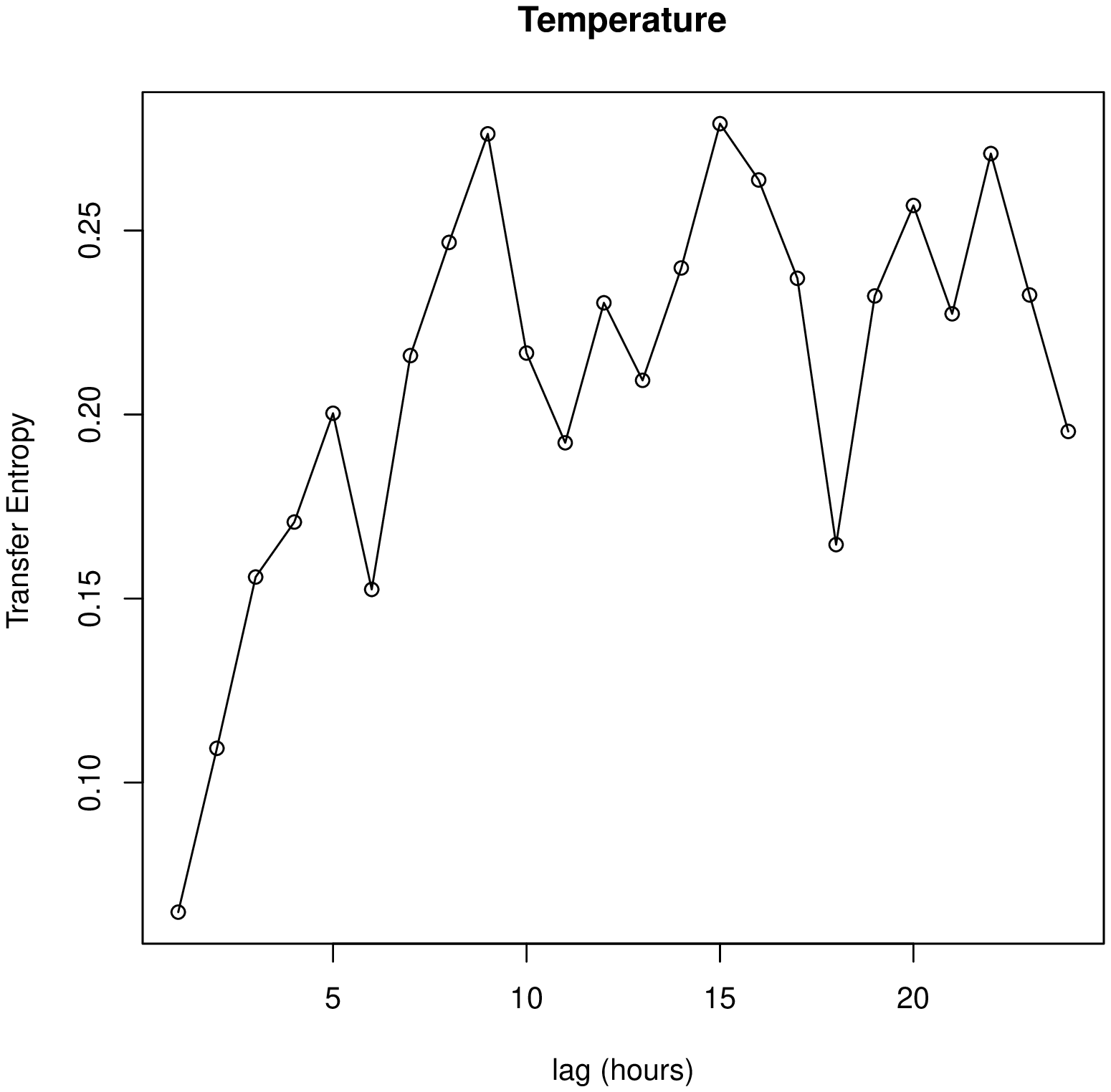}}
	\subfigure[Pressure]{\includegraphics[width=0.45\textwidth]{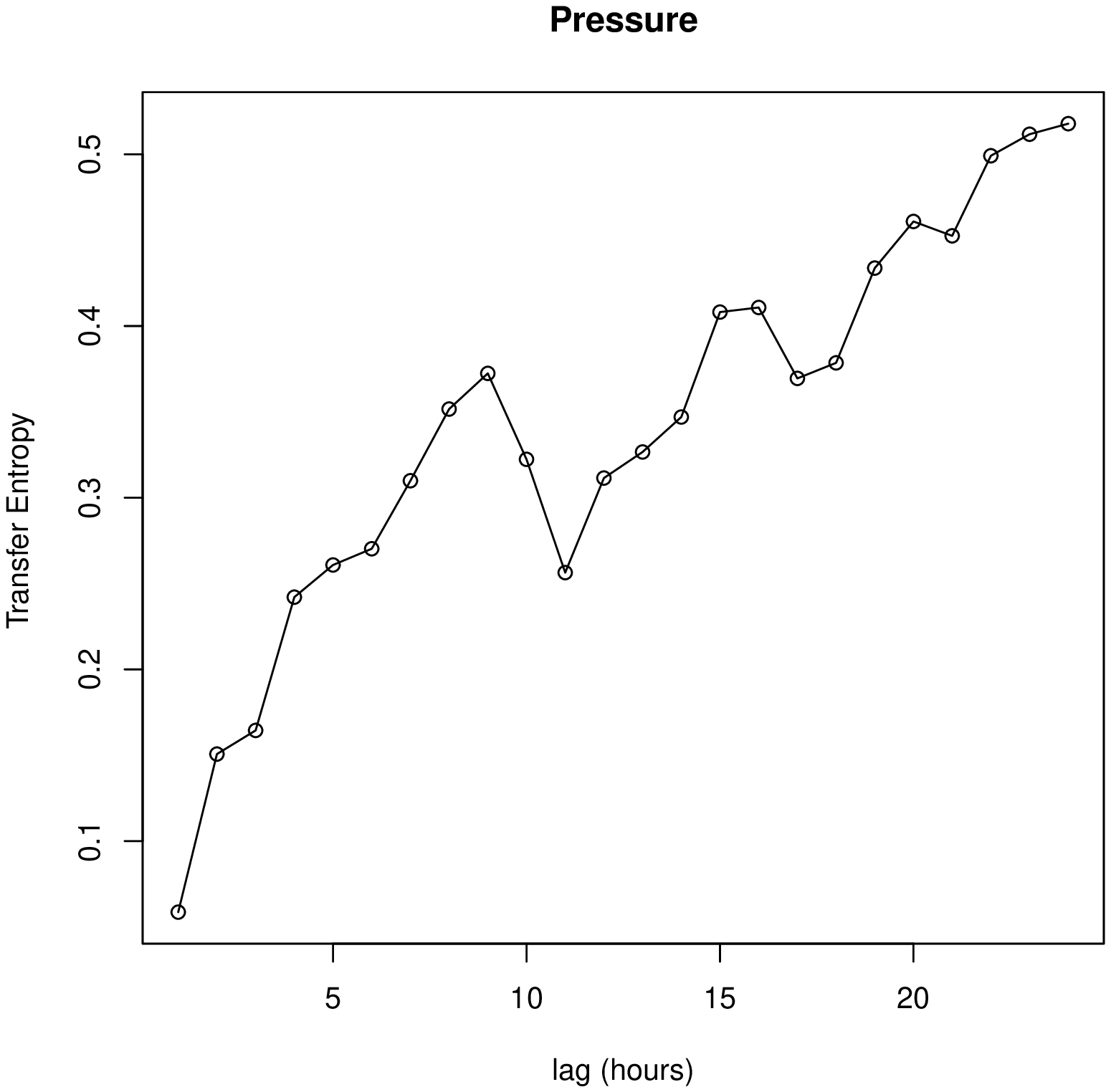}}
	\subfigure[Cumulated Wind Speed]{\label{f:ted}\includegraphics[width=0.45\textwidth]{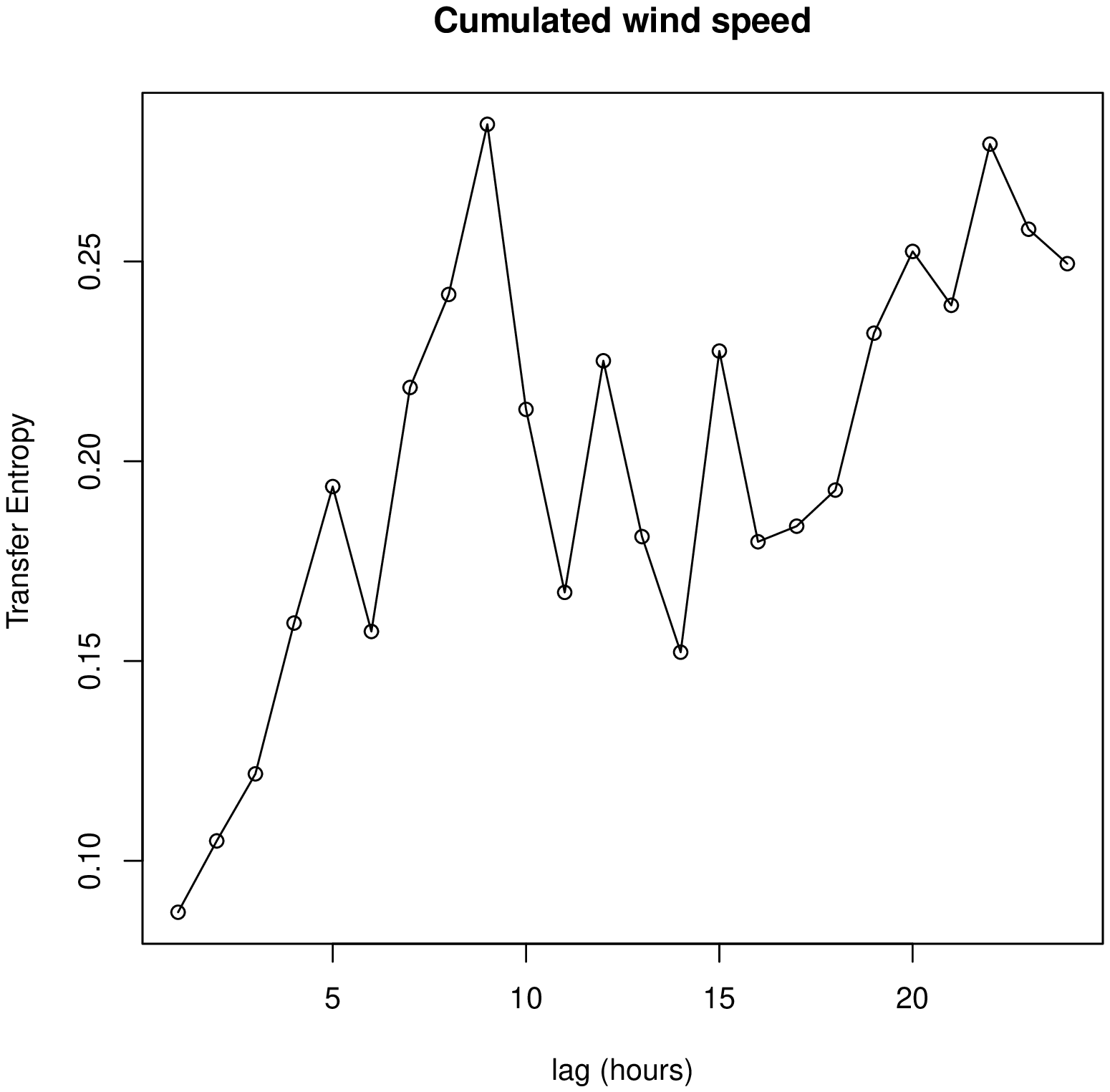}}	
	\caption{TE of different time lags of meteorological factors on PM2.5.}
	\label{f:te}
\end{figure}

The three CE factors of TE of temperature on PM2.5 are illustrated in Figure \ref{f:factors}. It can be learned that association strength measured by CE does not increase as causality strength measured by TE increases and that the increasing trend of TE is mostly contributed by the difference between joint CE and self joint CE.

\begin{figure}
	\centering
	\subfigure[TE]{\includegraphics[width=0.45\textwidth]{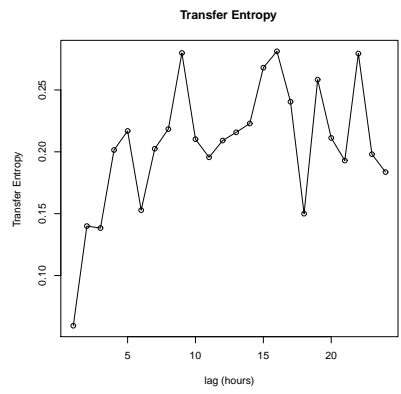}}
	\subfigure[Joint CE]{\includegraphics[width=0.45\textwidth]{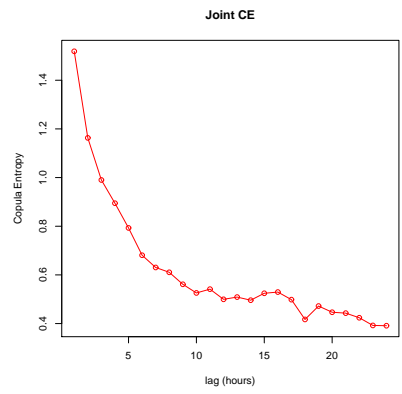}}
	\subfigure[Association]{\includegraphics[width=0.45\textwidth]{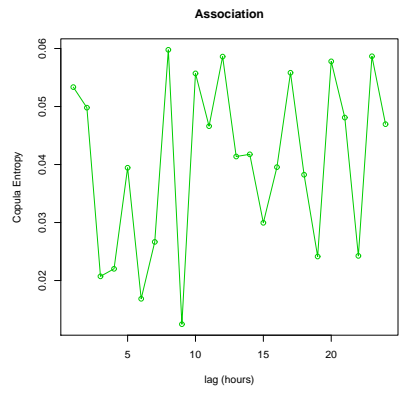}}
	\subfigure[Self Joint CE]{\includegraphics[width=0.45\textwidth]{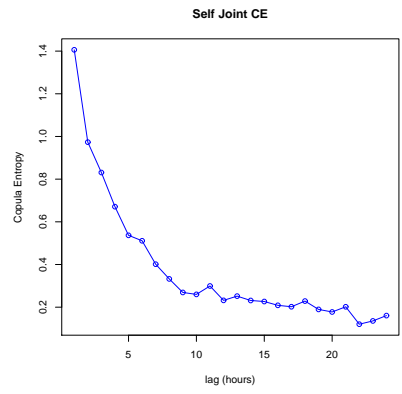}}	
	\caption{CE Factors of TE of temperature factor.}
	\label{f:factors}
\end{figure}

TE between meteorological factors are shown in Figure \ref{f:mfactors}. It can be learned that the influence of temperature on pressure and dew point increase with time and takes more than 10 hours lag to reach its peak and that the influence of cumulated wind speed on temperature and pressure increases very quickly in the first 4-5 hours lag.

\begin{figure}
	\centering
	\subfigure[Temperature to Pressure]{\includegraphics[width=0.45\textwidth]{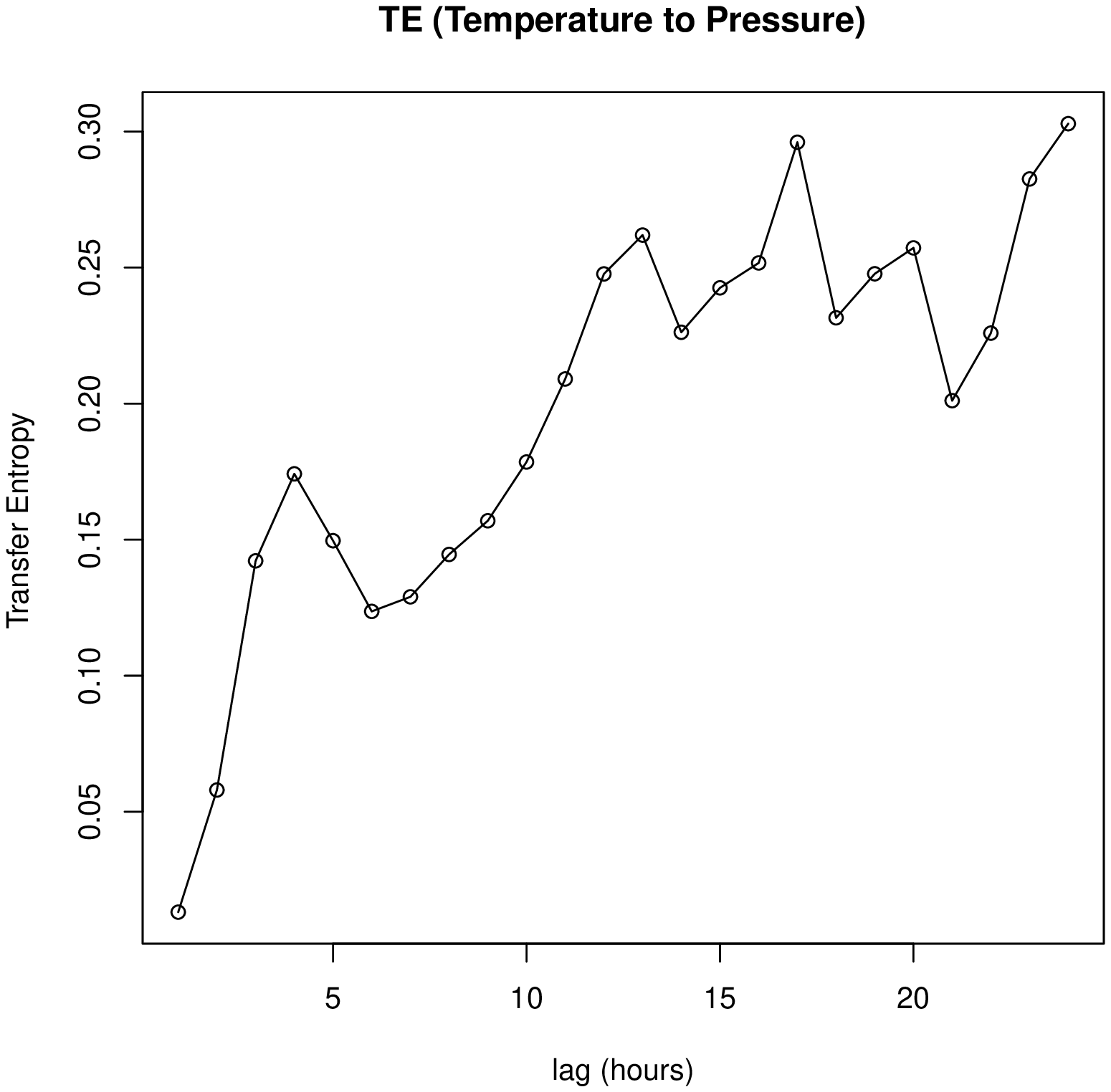}}
	\subfigure[Temperature to Dew Point]{\includegraphics[width=0.45\textwidth]{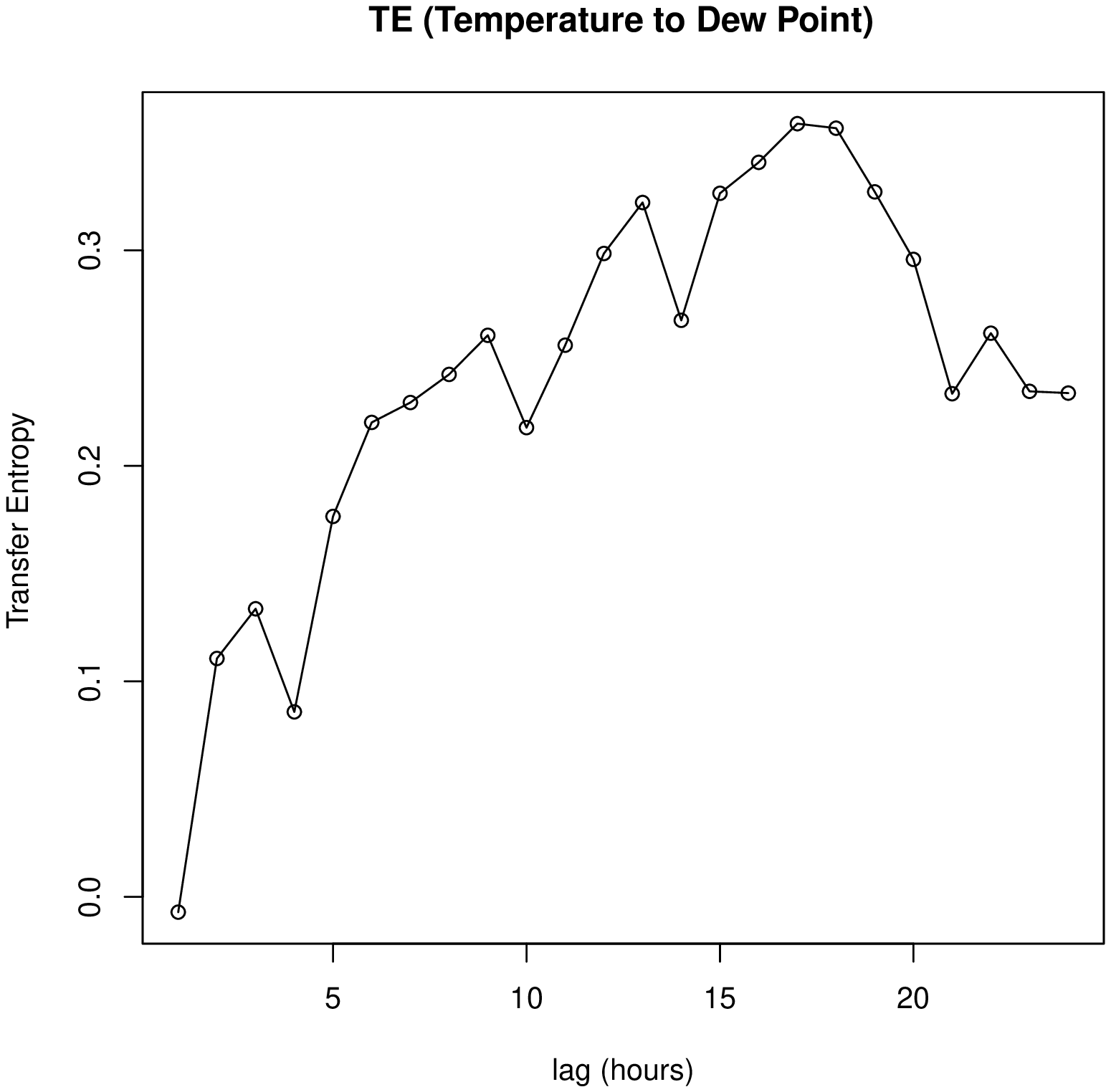}}
	\subfigure[Wind to Temperature]{\label{f:mfactorsc}\includegraphics[width=0.45\textwidth]{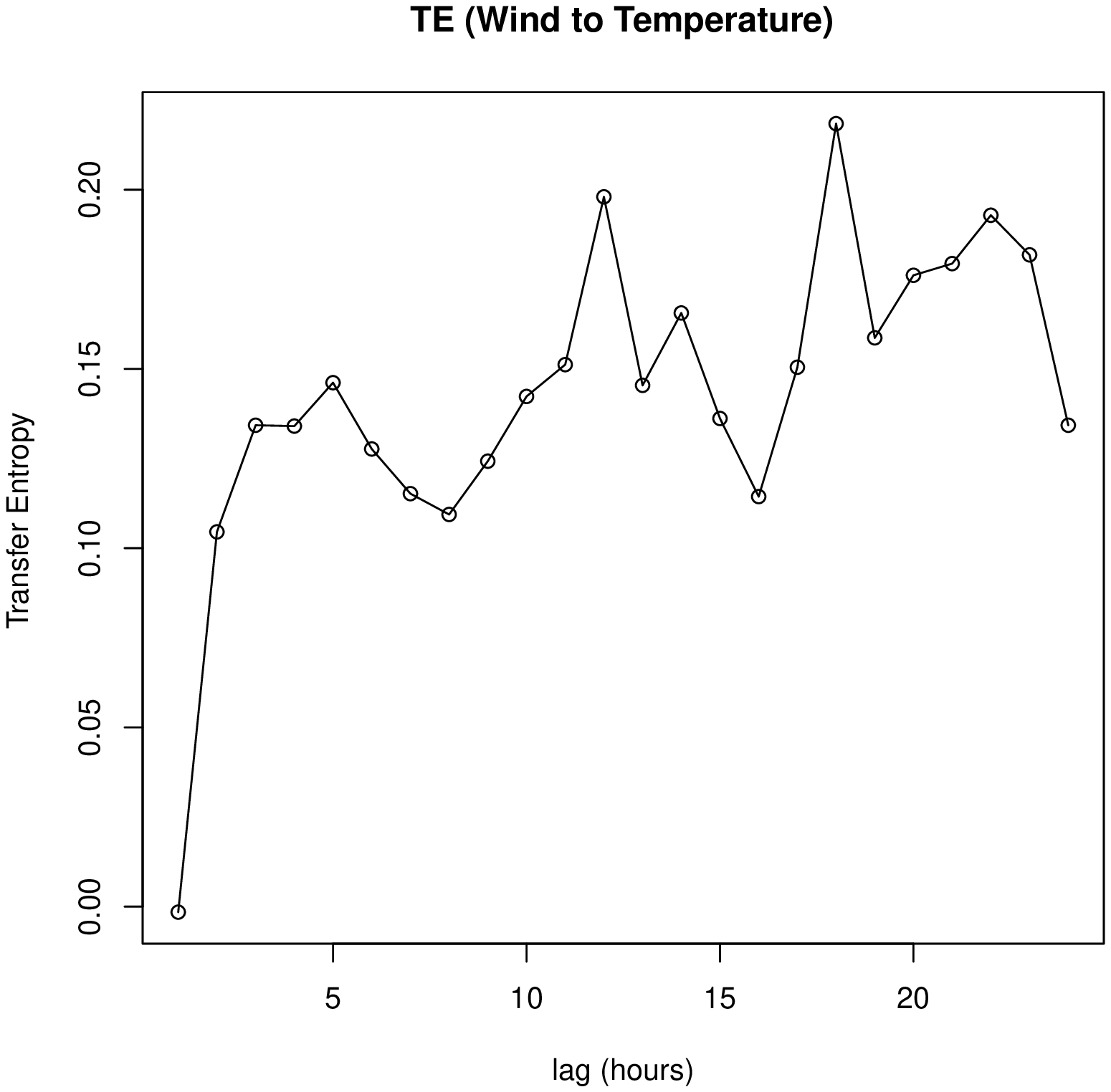}}
	\subfigure[Wind to Pressure]{\label{f:mfactorsd}\includegraphics[width=0.45\textwidth]{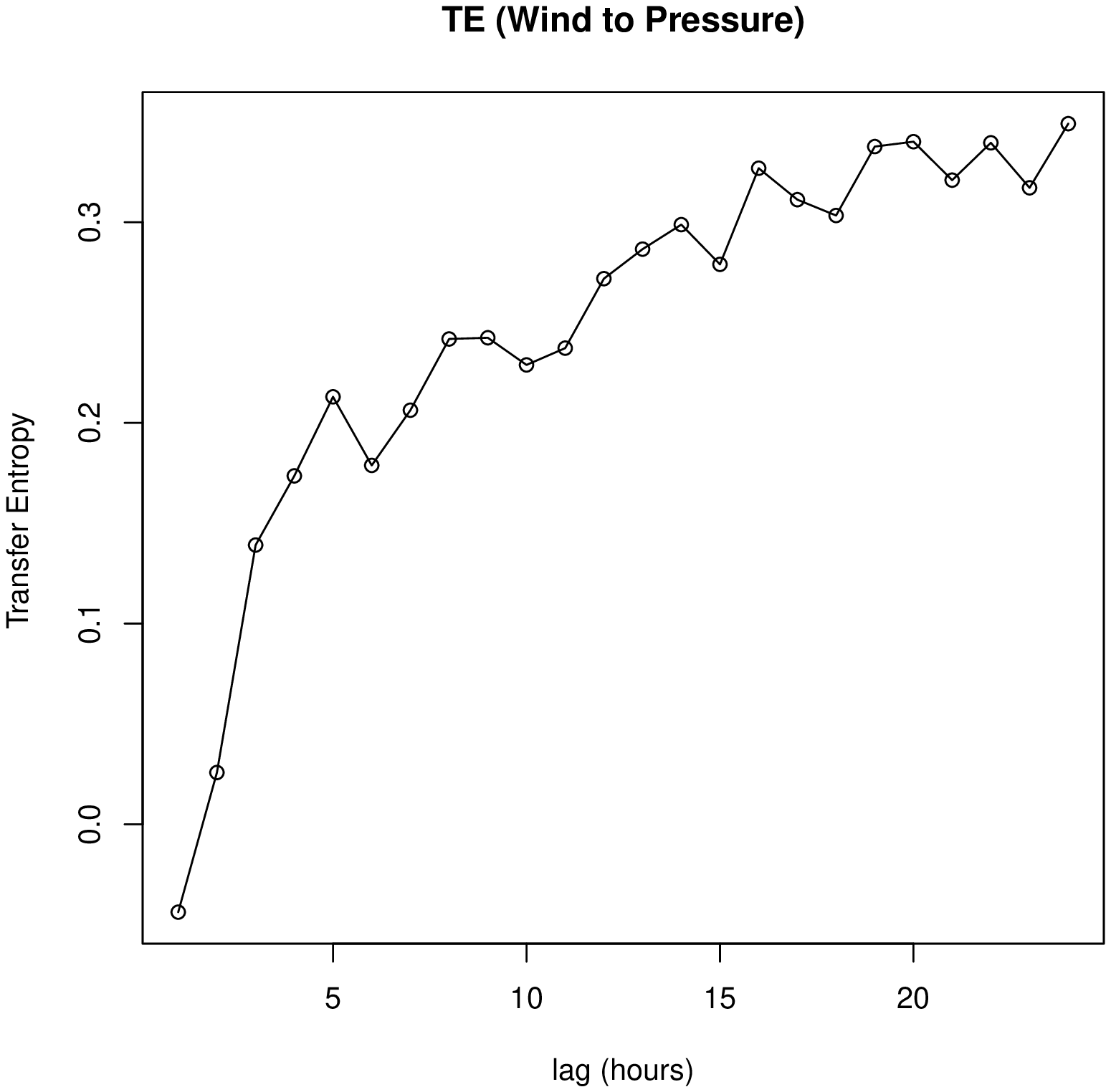}}	
	\caption{TE between meteorological factors.}
	\label{f:mfactors}
\end{figure}

The comparison between TE, KCI, and CDC on estimating causality from pressure to PM2.5 is shown in Figure \ref{f:3ci}. It can be learned that TE and CDC present similar results with an increasing phrase and a slow increasing phrase, while the result of KCI does not show such trend. 

\begin{figure}
	\centering
	\subfigure[TE via CE]{\includegraphics[width=0.78\textwidth]{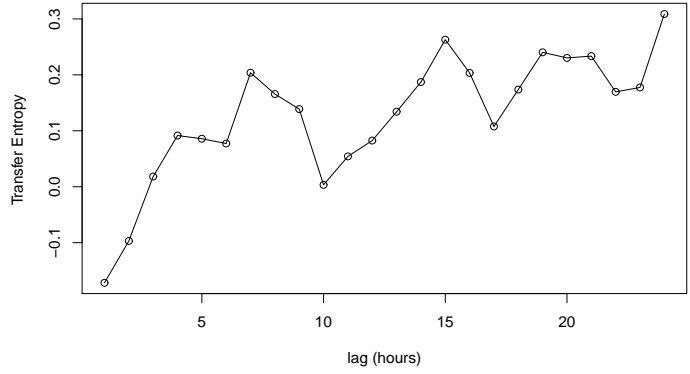}}
	\subfigure[CDC]{\includegraphics[width=0.78\textwidth]{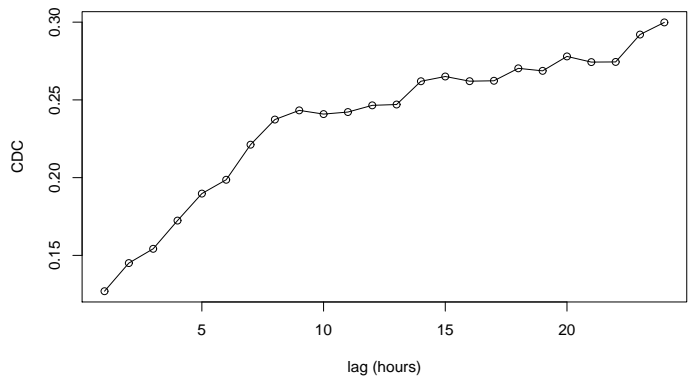}}
	\subfigure[KCI]{\includegraphics[width=0.78\textwidth]{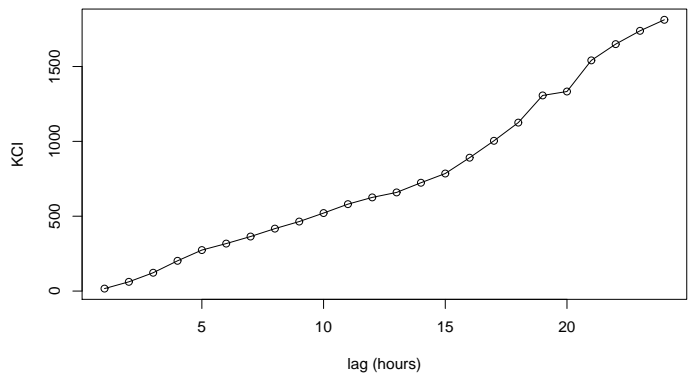}}
	\caption{Estimated TE, CDC, and KCI between pressure and PM2.5.}
	\label{f:3ci}
\end{figure}

\section{Discussion}
In this paper, we developed a theory of TE representation with only CE. It is essentially a theory of testing CI through CE. In the previous research \cite{ma2011mutual}, we have defined CE as a measurement of statistical independence. Therefore, a simple and elegant framework on testing both independence and CI based on only CE is proposed. Based on this theoretical framework, (conditional) independence can be tested by estimating CE. Since (conditional) independence is of fundamental importance in statistics \cite{dawid1979conditional}, such theoretical framework and estimation method will have wide applications in the related fields. 

Previously, two similar frameworks for testing (conditional) independence were also proposed based on kernel tricks in machine learning \cite{gretton2007nips,zhang2011uai} and distance covariance/correlation \cite{szekely2007measuring,szekely2009brownian,wang2015conditional}.  Compared with these two frameworks, the framework based on CE is much sound theoretically due to the rigorous definition of CE and much efficient computationally due to the simple method for estimating CE. As shown in Figure \ref{f:3ci}, our method can estimate TE or CI much effectively compared with its competitors.

The CE based representation of TE \eqref{eq:three1} casts theoretical insights on how causality is measured by attaching each term with causal meaning. The joint CE can be interpreted as measuring all the causal effect on $Y$ from two time series; the self joint CE as measuring the effect of the past of $Y$ on the future $Y$ which corresponds to causal dynamical mechanism of $Y$; the association as the causality by the possible common causes which has effect on both the cause $X$ and the past of the effect $Y$. In a word, TE can be interpreted as the difference between all the effects on $Y$ and the effects of all the factors except $X$ on $Y$, i.e., the effect of only $X$ on $Y$. 

Compared with the previous works on estimating TE, the proposed method is more computationally efficient. The previous works based on definitions or partial copula, usually require estimating (conditional) CDF or PDF explicitly, which is computationally unstable, especially in the cases of small data and high dimensions. Our method is based the non-parametric CE estimation, which do not need to estimate (conditional) PDF and thus has good convergence guarantee. Similar to ours, the previous works on estimating TE with kNN method \cite{kraskov2004estimating} is also non-parametric method. The difference between us is that our method is supported by elegant proposition theoretically and therefore computationally efficient since it needs estimate three entropy terms compared with four entropy terms in those previous works. 

The application of TE requires time series to be stationary. In our experiment, the data used was collected from Beijing during April to May. Considering the weather condition of this period at Beijing, the stationary of the data can be assumed to be true. Another assumption by TE is Markovianity. In our experiments, only one past state (1 hour lag) is conditioned, which means assuming the Markovianity of meteorological processes at local scale (mesogamme scale) \cite{seaman2000meteorological}. Considering the distance between two locations\footnote{The straight-line distance from US Embassy to Beijing Capital International Airport is about 17.7km.} where the data were collected, we argue that this assumption is reasonable at this temporal and spatial scale in meteorological sense. To validate this assumption, the TE of meteorological factors conditioned on the past four hours lag are also estimated, as shown in Figure \ref{f:te4}. Comparison between Figure \ref{f:te} and Figure \ref{f:te4} shows that conditioning on more hours lag does not change TE too much and the trend of TEs remain same, which suggests that Markovianity is a reasonable assumption for the experiments. 

\begin{figure}
	\centering
	\subfigure[Dew Point]{\includegraphics[width=0.45\textwidth]{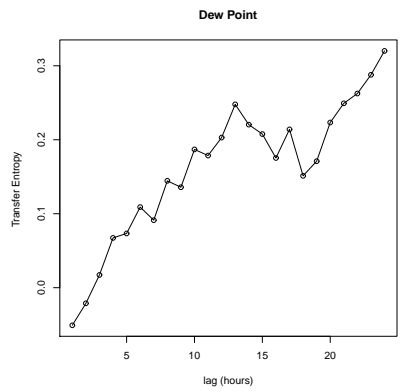}}	
	\subfigure[Temperature]{\includegraphics[width=0.45\textwidth]{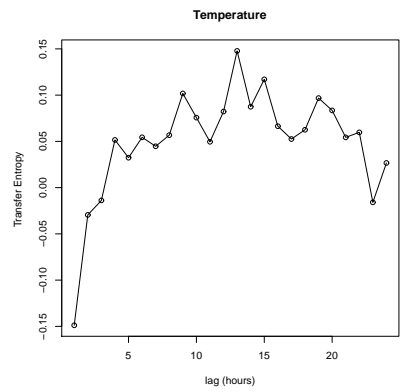}}
	\subfigure[Pressure]{\includegraphics[width=0.45\textwidth]{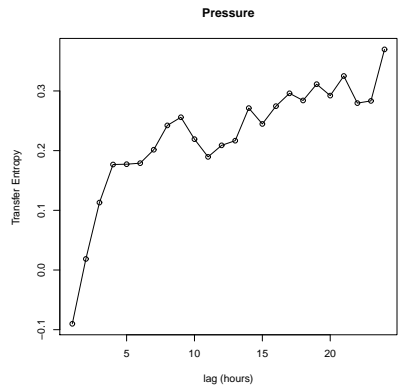}}
	\subfigure[Cumulated Wind Speed]{\includegraphics[width=0.45\textwidth]{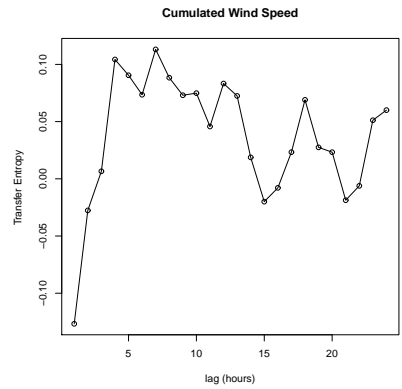}}	
	\caption{TE of different time lags of meteorological factors conditioned on past four hours states.}
	\label{f:te4}
\end{figure}

The experimental results show that the effects of meteorological factors on PM2.5 increase with roughly two phrases: the sharp increasing phrase in the first 9 hours time lag with its peak at about 9 hours lag and the flat increasing phrase during which TE of Dew point and pressure increase with relatively flat rate while TE of temperature and cumulated wind speed does increase any more. This phenomenon may means that the effects of meteorological factors on PM2.5 do not show immediately and is a cumulating meteorological process and that they affect the air quality of 9 hours later most. We conjecture that this corresponds to a underlying dynamical mechanism of PM2.5 generation.

The experimental results also show clearly how meteorological factors affects with each other. For example, It can be learned from Figure \ref{f:mfactors} that wind change temperature and pressure very quickly at Beijing. More specifically, wind changes temperature in 3 hours later and changes pressure in 5 hours later. Compared Figure \ref{f:ted} with Figure \ref{f:mfactorsc} and \ref{f:mfactorsd}, it can be learned that wind has causal effect on temperature and pressure hours more quickly than on PM2.5. This may be explained as how the entrained wind air is blended throughout the mixed layer and may help to build meteorological forecasting models for air quality \cite{seaman2000meteorological}.

The above experimental results help to understand the data with reasonable explanations with reference to meteorological knowledge. This means the proposed method can estimate TE effectively to infer causality relationships from observational data.

The results also show that association and causality are two different things. It can be learned from Figure \ref{f:factors} that even when the association between temperature and PM2.5 does not increase the TE of them still increases clearly. This suggests that only association (or correlation) is not suitable for investigate the causality relationship between temporal factors.

There are several research on analyzing air pollution data with causality tools. Dahlhaus and Eichler \cite{dahlhaus2003causality} tried to infer GC within air pollution data. However, instead of GC, the tool for measuring causal relationship is partial correlations on time series, which makes implicit Gaussian assumptions and infer linear relationship only. Zhu et al. \cite{zhu2016gaussian} applied TE to analyze the spatial-temporal causality relationships of air pollutants at different locations including Beijing. However, They estimated TE under also the Gaussian assumption, which is unreasonable both theoretically and empirically due to the non-linearity and non-Gaussianity of weather system \cite{sura2005multiplicative}. Compared with them, our estimation method makes no assumption on the underlying distributions and therefore derives more reliable results. Another related work by Kreuzer et al. \cite{kreuzer2019bayesian} applied copula to extend Gaussian state space model for predicting air pollution with the same dataset as ours. However, even introducing copula to state space model extends model flexibility for the issues of non-linearity and non-Gaussianity, state space model is still questionable for the dynamics of the underlying atmospheric system. For example, the meaning of the state variables is unexplainable in theory and unobservable in practice. Meanwhile, selecting parametric copula model brings the risk of model misspecification.

\section{Conclusion}
In this paper, we prove that TE can be represented with only CE and then propose a non-parametric method for estimating TE via CE which composed of only two simple steps. The proposed method was applied to analyze the Beijing PM2.5 data in the experiments. Experimental results show that the proposed method can identify causality relationships from data, discover how meteorological factors affects PM2.5 and each other, and hence help to understand the data better and to build better forecasting model for time series data. The experiments that compare the proposed method with other methods on testing CI show the advantage of our method.

\bibliographystyle{unsrt}
\bibliography{ci}

\end{document}